\documentclass[11pt,a4paper]{article}
\PassOptionsToPackage{hyphens}{url}\usepackage[hyperref]{emnlp2020} %
\usepackage{footmisc}
\usepackage{wrapfig}
\usepackage{times}
\usepackage{latexsym}
\usepackage{algorithm}
\usepackage{algorithmic}
\usepackage{xcolor}
\usepackage{bm}
\usepackage{url}
\usepackage{setspace}
\usepackage{xspace}
\usepackage{amsmath}
\usepackage{amsfonts}
\usepackage{booktabs}
\usepackage{graphicx}
\usepackage{amsmath,amsfonts,amssymb}
\usepackage{enumitem}
\usepackage{etoolbox}
\usepackage{amsthm}
\usepackage{soul}
\usepackage{setspace}
\usepackage{bibunits} %
\usepackage{tikz}
\usetikzlibrary{tikzmark}

\AtBeginEnvironment{quote}{\singlespacing\small}

\newtheorem{claim}{Claim}
\newtheorem*{claim*}{Claim}

\definecolor{emapcolor}{RGB}{230,97,1}

\DeclareMathOperator*{\argmin}{arg\,min}

\aclfinalcopy %

\newcommand{\titletext}{Does 
my multimodal model learn cross-modal interactions?\\
  It's %
  harder to tell than you might think!}

\newcommand{\mparagraph}[1]{\noindent\textbf{{#1}}.}
\newcommand{\mparagraphnp}[1]{\noindent\textbf{{#1}}}

\newcommand{\emap}{\textsf{EMAP}\xspace}
\newcommand{\emaps}{\textsf{EMAPs}\xspace}
\newcommand{\emaped}{\textsf{EMAPed}\xspace}

\makeatletter
\renewcommand\sectionautorefname{\S\@gobble}

\newif\iffinal

\finalfalse
\iffinal
\newcommand{\jack}[1]{}
\newcommand{\llee}[1]{}
\else
\newcommand{\jack}[1]{\textcolor{blue}{\textbf{Jack: #1}}}
\newcommand{\llee}[1]{\textcolor{red!55!blue}{\textbf{LL: #1}}}
\fi

\newcommand{\tparam}{\tau_i}
\newcommand{\vparam}{\phi_j}

\title{\titletext}

\author{
  Jack Hessel \\
  Allen Institute for AI \\
  {\tt jackh@allenai.org} \\\And
  Lillian Lee \\
  Cornell University \\
  {\tt llee@cs.cornell.edu}
}

\date{}

\begin{document}
\maketitle
\begin{abstract}

Modeling expressive cross-modal interactions 
seems crucial in multimodal tasks,
such as visual question answering.
However, sometimes high-performing black-box algorithms turn out to be
mostly exploiting unimodal signals in the data.
 We propose a new diagnostic tool,
\emph{empirical multimodally-additive function projection} (\emap),
for isolating whether or not cross-modal {interactions} improve
performance for a given model on a given task. This function
projection modifies model predictions so that cross-modal interactions
are eliminated, isolating the additive, unimodal structure. 
For seven \mbox{image+}text classification tasks (on each of which we set new
state-of-the-art benchmarks), we find that, in many cases,
removing cross-modal interactions results in little to no performance
degradation. Surprisingly, this holds even when expressive models, with capacity to consider
interactions, otherwise outperform less expressive models; thus, performance
improvements, even when present, often cannot be attributed to consideration of cross-modal
feature interactions.
We hence recommend
that researchers in multimodal machine learning 
report the performance not only of unimodal baselines, but also
the \emap of their best-performing model.

\end{abstract}

\begin{bibunit}[acl_natbib]
\section{Introduction}

\label{sec:intro}
Given the presumed importance of reasoning across modalities in multimodal
machine learning tasks, we should
evaluate a model's ability to leverage  cross-modal interactions.  
But such evaluation is not straightforward; 
for example, an early Visual Question-Answering (VQA) challenge was later ``broken'' 
by a high-performing method that ignored the image entirely \citep{jabri2016revisiting}.

One response is to create multimodal-reasoning datasets that are specifically 
and cleverly balanced 
to resist language-only or visual-only models; examples are VQA 2.0
\cite{goyal2017making}, NLVR2 \cite{suhr2018corpus}, and GQA
\cite{hudson2018gqa}. However, 
a balancing approach
not always desirable. 
For example, if image+text data is collected from an online social
network (such as for popularity prediction or sentiment analysis),
post-hoc rebalancing may obscure trends in the original data-generating
processs. So,
what alternative diagnostic tools are available for
better understanding what models learn?

\begin{figure}
  \centering
  \includegraphics[width=.95\linewidth]{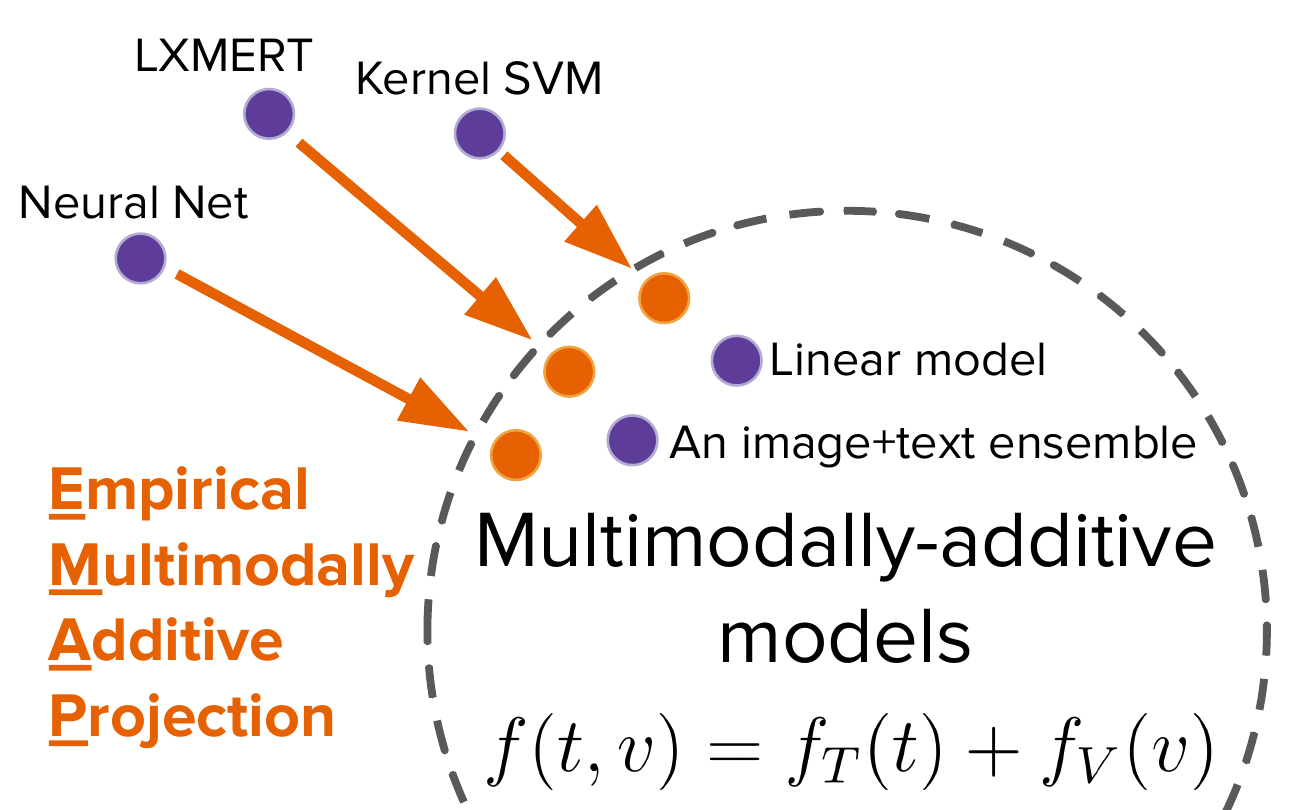}
  \caption{We introduce \emap, a diagnostic for classifiers that take
    in $\underline{t}$extual and $\underline{v}$isual inputs. Given a
    (black-box) trained model, \emap computes the predictions of an
    image/text ensemble that best approximates the full model predictions via
    empirical function projection.
    Although the projected predictions 
    lose visual-textual interaction signals
    exploited by the original model, 
    they often perform suprisingly well.}
  \label{fig:illustration}
\end{figure}

The main tool utilized by prior work
is
\emph{model comparison.} In addition to comparing against text-only
and image-only baselines, often, two multimodal models with differing
representational capacity (e.g., a cross-modal attentional neural
network vs. a linear model) are trained and their performance
compared. The argument commonly made is that if model A, with greater
expressive capacity, outperforms model B, then the performance
differences can be at least partially attributed to that increased
expressivity.

{But is that a reliable argument?} Model performance comparisons
are an opaque tool for analysis, especially for deep neural networks:
performance differences versus baselines, frequently small in
magnitude, can often be attributed to hyperparameter search schemes,
random seeds, the number of models compared, etc. \citep{yogatama2015bayesian,dodge-etal-2019-show}. Thus, while
model comparisons are an acceptable starting point for demonstrating
whether or not a model is learning an interesting set of (or any!)
cross-modal factors, they provide rather indirect evidence.

We propose \emph{\underline{E}mpirical
  \underline{M}ultimodally-\underline{A}dditive\footnote{In
    \autoref{sec:sec_with_emap}, we more formally introduce
    \emph{additivity}.} function \underline{P}rojection} (\emap) as an
additional diagnostic for analyzing multimodal classification
models. Instead of comparing two different models, a single multimodal
classifier's predictions are \emph{projected} onto a less-expressive space: the
result is 
equivalent to 
a set of predictions made by the closest possible ensemble of text-only and
visual-only classifiers.
The projection process is computationally efficient, has no
hyperparameters to tune, can be implemented in a few lines of code, is
provably unique and optimal, and works on any multimodal classifier:
we apply it to models ranging from polynomial kernel SVMs to deep,
pretrained, Transformer-based self-attention models.

We first verify that \emaps do degrade performance for synthetic cases
and for visual question answering cases where datasets have been
specifically designed to require cross-modal reasoning. But we then
examine a test suite of several recently-proposed multimodal
prediction tasks that have \emph{not} been specifically balanced in
this way. We first achieve state-of-the-art performance for all of the
datasets using a linear model.
Next, we examine more expressive interactive models, e.g., pretrained
Transformers, capable of cross-modal attention. While these models
sometimes outperform the linear baseline, \emap reveals that
performance gains are (usually) not due to multimodal
interactions being leveraged.

\mparagraphnp{Takeaways:} For future work on multimodal classification
tasks, we recommend authors report the performance of: 1) unimodal
baselines; 2) any multimodal models they consider; and,
\emph{critically,} 3) the \emph{empirical multimodally-additive
  projection} (\emap) of their best performing multimodal model (see
\autoref{sec:faq} for our full recommendations).

\section{Related Work}

\mparagraph{Constructed multimodal classification tasks}
In addition to image question
answering/reasoning datasets
already mentioned in \autoref{sec:intro}, other multimodal tasks have
been constructed, e.g., video QA \cite{lei2018tvqa,zellers2019vcr},
visual entailment \cite{xie2018visual}, hateful multimodal meme
detection \cite{kiela2020hateful}, and tasks related to visual dialog
\cite{de2017guesswhat}.
In these cases, unimodal baselines are shown to achieve lower
performance relative to their expressive multimodal counterparts.

\mparagraph{Collected multimodal corpora} Recent computational work
has examined diverse multimodal corpora collected from in-vivo social
processes, e.g., visual/textual advertisements
\cite{hussain2017automatic,ye2018advise,zhang2018equal}, images with
non-literal captions in news articles \cite{weiland2018knowledge}, and
image/text instructions in cooking how-to documents
\cite{alikhani2019cite}. In these cases, multimodal classification
tasks are often proposed over these corpora as a means of testing
different theories from semiotics \citep[inter
  alia]{barthes1988image,o1994language,lemke1998multiplying,o2004multimodal};
unlike many VQA-style datasets, they are generally not specifically
balanced to force models to learn cross-modal interactions.

Without rebalancing, should we expect cross-modal interactions to be
useful for these multimodal communication corpora? Some semioticians
posit: \emph{yes!} 
\emph{Meaning multiplication} \cite{barthes1988image} between images
and text suggests, as summarized by \newcite{bateman2014text}:
\par
\begin{quote}
under the right conditions, the value of a
combination of different modes of meaning can be worth more than the
information (whatever that might be) that we get from the modes when
used alone. In other words, text `multiplied by' images is more than
text simply occurring with or alongside images.
\end{quote}
\newcite{jones1979interaction} provide experimental evidence of
conditional, compositional 
{interactions} between image and text
in a humor setting, concluding that ``it is the dynamic interplay
between picture and caption that describes the multiplicative
relationship'' between modalities. Taxonomies of the specific types of
compositional relationships image-text pairs can exhibit have been
proposed \cite{marsh2003taxonomy,martinec2005system}.

\mparagraph{Model Interpretability} In contrast to methods that design
more interpretable algorithms for prediction
\cite{lakkaraju2016interpretable,ustun2016supersparse}, several researchers
aim to explain the behavior of complex, black-box predictors on
individual instances
\cite{kononenko2010efficient,ribeiro2016model}. The most related of
these methods to the present work is \newcite{ribeiro2018anchors}, who
search for small sets of ``anchor'' features that, when fixed, largely
determine a model's output prediction on the input points. While
similar in spirit to ours and other methods that ``control for'' a
fixed subset of features (e.g., \newcite{breiman2001random}), their
work 1) focuses only on high-precision, local explanations on
single instances; 2) doesn't consider multimodal models (wherein
feature interactions are combinatorially more challenging in
comparison to unimodal models); and 3) wouldn't guarantee
consideration of multimodal ``anchors.''

\section{EMAP}%
\label{sec:sec_with_emap}

\mparagraph{Background} We consider models $f$ that assign scores to
textual-visual pairs $(t, v)$, where $t$ is a piece of text (e.g., a
sentence), and $v$ is an image.\footnote{The methods introduced here
  can be easily extended beyond just text/image pairs (e.g., to
  videos, audio, etc.), and to more than two modalites.} In
multi-class classification settings, values $f(t, v) \in \mathbb{R}^d$ 
typically 
serve
as per-class scores.
In ranking settings $f(t_1, v_1)$
may be compared to $f(t_2, v_2)$ via a ranking objective.

We are particularly interested in the types of compositionality that
$f$ uses over its visual and textual inputs to produce scores.
Specifically, we 
distinguish between
\emph{additive} models \cite{hastie1987generalized}
and \emph{interactive} models
\cite{friedman2001greedy,friedman2008predictive}.\footnote{The
  multimedia commonly makes a related distinction between \emph{early
    fusion} (joint multimodal processing) and \emph{late fusion}
  (ensembles) \cite{snoek2005early}.} 
\newcommand\allindices{{\cal I}}
A function $f$ of a 
representation $z$ of
$n$ input features is  \emph{additive} in $I$ if it decomposes as
$$f(z) = f_I(z_I) + f_{\setminus I}(z_{\setminus I}), $$ 
where (setting $\allindices = \{1, \ldots, n\}$ for convenience)
$I \subset \allindices$ indexes a subset of the features,
$\setminus I = \allindices\setminus I$, and for any $S \subset \allindices$, 
$z_S$ is the restriction of $z$ to only those features whose indices appear in $S$. 

In our case,
because features are the union of \emph{\underline{T}}extual and
\emph{\underline{V}}isual predictors, we say that a model is
\emph{multimodally additive} if it decomposes as
\begin{equation}
  f(t, v) = f_T(t) + f_V(v)
  \label{eq:multimodally_additive}
\end{equation}
  
Additive multimodal models are simply ensembles of unimodal
classifiers and as such, may be considered underwhelming to multiple
communities.
A recent ACL paper, for example, refers to 
such ensembles as ``naive.''
On the semiotics side, the
conditionality implied by meaning multiplication
\cite{barthes1988image} --- that the joint semantics of an
image/caption depends non-linearly on its accompaniment --- 
cannot be modeled additively: multimodally additive models posit that
each image, independent of its text pairing, contributes a fixed score
to per-class logits (and vice versa).

In contrast, \emph{multimodally interactive} models are the set of
functions that \emph{cannot} be decomposed 
as in \autoref{eq:multimodally_additive} --- that
is, $f$'s output conditionally depends on its inputs in a
non-additive fashion.

\mparagraph{Machine learning models} One canonical multimodally
additive model is a linear model trained over a concatenation of
textual and visual features 
$[t;v]$, i.e.,
\begin{equation}
f(t, v) = w^T[t;v] + b = \underbrace{w_t^T t}_{f_T(t)} +
\underbrace{w_v^T v + b}_{f_V(v)}.
\label{eq:linear_model}
\end{equation}
We later detail several multimodally interactive models, including
multi-layer neural networks, polynomial kernel SVMs, pretrained
Transformer attention-based models, etc. However, even though
interactive models are \emph{theoretically capable} of modeling a more
expressive set of relationships, it's not clear that they will learn
to exploit this additional expressivity, particularly when simpler
patterns suffice.

\mparagraph{Empirical multimodally-additive projections: \emap} Given
a fixed, trained model $f$ (usually one theoretically capable of modeling
visual-textual interactions) and a set of $N$ labelled 
datapoints $\{(v_i, t_i, y_i)\}_{i=1}^N$, we aim to answer: \emph{does
  f utilize cross-modal feature interactions to produce more accurate
  predictions for these datapoints?}

\begin{algorithm}[t] \small
  \newcommand\preds{{\rm preds}_{\widehat f}}
  \newcommand\proj{{\rm proj}}

  \begin{algorithmic}
    \STATE {\bfseries Input:} a trained model $f$ that outputs logits; a set of text/visual pairs ${\cal D} = \{(t_i, v_i)\}_{i=1}^N$
    \STATE {\bfseries Output:} the predictions of $\hat f$, the empirical 
    projection of $f$ onto the set of multimodally-additive functions, on ${\cal D}$.
    \STATE
    \STATE $f_{cache} = 0_{N\times N\times d}; \preds = 0_{N\times d}$
    \FOR{$i, j \in \{1,2,\dots, N\} \times \{1,2,\dots, N\}$}
    \STATE $f_{cache}(i, j) = f(t_i, v_j)$
    \ENDFOR
    \STATE $\hat{\mu} \in \mathbb{R}^d = \frac{1}{N^2}\sum_{i,j} f_{cache}(i, j)$
    \FOR{$i \in \{1,2,\dots, N\}$}
    \STATE $\proj_t = \frac{1}{N} \sum_{j=1}^N f_{cache}(i, j)$
    \STATE $\proj_v = \frac{1}{N} \sum_{j=1}^N f_{cache}(j, i)$
    \STATE $\preds[i] = \proj_t + \proj_v - \hat{\mu}$
    \ENDFOR
    \STATE {\bfseries return} $\preds$
  \end{algorithmic}
  \caption{{\small Empirical Multimodally-Additive Projection (EMAP); worked example in  \autoref{sec:walkthrough}.}}
  \label{alg:additive_projection}
\end{algorithm}

\newcite{hooker2004discovering} gives a general method for computing
the 
projection of a function $f$ onto a set of
functions with a specified ANOVA structure (see also
\newcite{sobol2001global,liu2006estimating}): our algorithmic
contributions are to extend the method to multimodal models with $d>1$
dimensional outputs, and to prove that the multimodal empirical
approximation is optimal.  Specifically: we are interested in 
$\tilde
f$, the following projection of $f$ onto the set of multimodally-additive
functions:
\begin{equation*} \tilde f(t,v) = \underbrace{\mathop{\mathbb{E}}_v[f(t, v)]}_{f_T(t)}
+ \underbrace{\mathop{\mathbb{E}}_t[f(t, v)]}_{f_V(v)} -
\underbrace{\mathop{\mathbb{E}}_{t,v}[f(t, v)]}_{\mu}
\end{equation*}
where
$\mathop{\mathbb{E}}_v[f(t, v)]$, a function of $t$, is the \emph{partial
dependence} of $f$ on $t$, i.e., $f_T(t) = \mathop{\mathbb{E}}_v[f(t,
  v)] = \int f(t, v) p(v) dv$,\footnote{Both
  \newcite{friedman2001greedy} and \newcite{hooker2004discovering}
  argue that this expectation should be taken over the marginal $p(v)$
  rather than the conditional $p(v|t)$.} 
and similarly for the other expectations.
  An empirical approximation of
the partial dependence function for a given $t_i$ can be computed by
looping over all observations:
$$ \hat f_T(t_i) 
= \frac{1}{N} \sum_{j=1}^N f(t_i,
v_j) \,.$$
We similarly arrive at $\hat f_V(v_i)$ and $\hat \mu$, yielding
\begin{equation}
 \hat f(t_i,v_i) =  \hat f_T(t_i)  +  \hat f_V(v_i) + \hat \mu
 \label{eq:projection}
\end{equation}
which is what \emap, \autoref{alg:additive_projection},  computes for each $(t_i, v_i)$.
Note that \autoref{alg:additive_projection}
involves evaluating the original model $f$ on all $N^2$ $\langle v_i, t_j \rangle$ pairs --- even
mismatched image/text pairs that do not occur in the observed data. 
In
practice, we recommend only computing this projection over the
evaluation set.\footnote{\label{fn:prohibitively_large}
  When even restricting to the evaluation set would be too expensive, as in the case of the R-POP data we experiment with later, one can repeatedly run EMAP on randomly-drawn 500-instance (say) samples from the test set.}  Once the predictions of $f$ and $\hat f$
are computed over the evaluation points, then their performance can be
compared according to standard evaluation metrics, e.g., accuracy or
AUC.

In \autoref{sec:theoretical}, we prove that the
(mean-centered) sum of empirical partial dependence functions is
optimal with respect to squared error, that is:
\begin{claim*}
  Subject to the constraint that $\hat f$
  is multimodally
  additive, \autoref{alg:additive_projection} produces a unique
  and optimal solution to
  \begin{equation}
  \argmin_{\hat f\, {\rm values}} \sum_{i,j} \|f(t_i, v_j) - \hat f(t_i, v_j)\|_2^2.
  \label{eq:loss}
\end{equation}
\end{claim*}

\section{Sanity Check: \emap Hurts 
in Tasks that Require Interaction}

In \S\ref{sec:sec_where_emap_rocks}, we will see that \emap provides a
very strong baseline for ``unbalanced'' multimodal classification
tasks. But first, we first seek to verify that \emap degrades model
performance in cases that are designed to require cross-modal
interactions.

\subsubsection*{Synthetic data}
We generate a set of ``visual''/``textual''/label
data $(v, t, y)$ according to the following process:\footnote{We
  sample 5K points in an 80/10/10 train/val/test split with $\langle d,
  d_1, d_2, \delta \rangle = \langle 100, 2000, 1000, .25 \rangle$,
  though similar results were obtained with different parameter
  settings.}

\begin{small}
  \begin{enumerate}[leftmargin=*,topsep=0pt,itemsep=-1ex,partopsep=1ex,parsep=1ex]
\item Sample random projection $V \in \mathbb{R}^{d_1 \times
  d}$ and $T \in \mathbb{R}^{d_2 \times d}$ 
  from $U(-.5, .5)$.
\item Sample $v, t \in \mathbb{R}^d \sim N(0, 1)$; normalize
  to unit length.
\item If $|v \cdot t| > \delta$ proceed,
  else, return to the previous step.
\item If $v \cdot t > 0$, then $y=1$, else $y=0$.
\item Return the data point $(Vv, Tt, y)$.
\end{enumerate}
\end{small}

This function challenges models to learn whether or not the dot
product of two randomly sampled vectors in $d$ dimensions is positive
or negative --- a task that, by construction, requires modeling the
multiplicative interaction of the features in these vectors. To
complicate the task, the vectors are randomly projected to two
``modalities'' of different dimensions, $d_1$ and $d_2$ respectively. 

We trained a linear model, a polynomial
kernel SVM, and a feed-forward neural network on this data: the
results are in \autoref{tab:synthetic_results}.
The linear model is additive and thus incapable of learning any
meaningful pattern on this data. In contrast, the kernel SVM and
feed-forward NN, interactive models, are able to
fit the test data almost perfectly. However, when we apply \emap to the interactive models,
as expected, their performance drops to random.

\subsubsection*{Balanced VQA Tasks} Our next sanity check is to
verify that \emap hurts the performance of interactive models on two
real multimodal classification tasks that are specifically balanced to
require modeling cross-modal feature interactions: VQA 2.0
\cite{goyal2017making} and GQA \cite{hudson2018gqa}.

First, we fine-tuned LXMERT \cite{tan2019lxmert}, a multimodally-interactive,
pretrained, 14-layer Transformer model (See \autoref{sec:sec_with_models} 
for full description)  that achieves SOTA on both
datasets. 
The LXMERT authors frame question-answering
as a multi-class image/text-pair classification problem --- 3.1K candidate answers
for VQA2, and 1.8K for GQA. 
In Table~\ref{tab:qa_results}, we compare, in cross-validation, the means of: 1) accuracy
of LXMERT, 2) accuracy of the \emap of LXMERT, and 3) accuracy  of simply predicting the most common answer for all
questions. As
expected, \emap decreases accuracy on VQA2/GQA by 30/19 absolute
accuracy points, respectively: this suggests LXMERT is utilizing
feature interactions to produce more accurate predictions on these
datasets. On the other hand, performance of LXMERT's \emap remains
substantially better than constant prediction, suggesting that LXMERT's logits do nonetheless
leverage some unimodal signals in this data.\footnote{\emaped
  LXMERT's performance is comparable to LSTM-based, text-only models,
  which achieve 44.3/41.1 accuracy on the full VQA2/GQA test set,
  respectively.}

\begin{table}
  \centering
  \begin{tabular}{lccc}
    \toprule
    & Linear (A) & Poly (I) & NN (I) \\
    \midrule
    Test Acc & 52.8\% & 99.6\% & 99.0\% \\
    \rotatebox[origin=c]{180}{$\Lsh$} + \emap & 52.8\% & 49.4\% & 53.8\% \\
    \bottomrule
  \end{tabular}
  \caption{Prediction accuracy on synthetic dataset using additive (A)
    models, interactive (I) models, and their \emap
    projections. Random guessing achieves 50\% accuracy. Under
    \emap, the interactive models degrade to (close to) random, 
    as desired.
    See \autoref{sec:sec_with_models} for training details.
    }
  \label{tab:synthetic_results}
\end{table}

\begin{table}
  \centering
  \begin{tabular}{lccc}
    \toprule
    & LXMERT & $\rightarrow$\emap & Const Pred \\
    \midrule
    VQA2 & 70.3 & 40.5 & 23.4 \\
    GQA & 60.3 & 41.0 & 18.1 \\
    \bottomrule
  \end{tabular}
  \caption{As expected, for VQA2 and GQA, the mean accuracy of LXMERT is substantially higher than its empirical multimodally
    additive projection (EMAP). Shown are
    averages over $k=15$ random subsamples of 500 dev-set instances.}
  \label{tab:qa_results}
\end{table}

\section{``Unbalanced'' Datasets + Tasks}

\begin{table*}
  \centering
  \resizebox{.95\textwidth}{!}{
    
\newcommand{\convertto}[2]{#2} %
{\small
\begin{tabular}{lllr}
  \toprule
  Original paper & Task (structure) & Abbrv. & \# image+text  pairs we recovered \\
  \midrule
  \newcite{kruk2019integrating} & Instagram & & \\
  & \rotatebox[origin=c]{180}{$\Lsh$} intent (7-way classification) & I-INT & \convertto{1299}{1.3K} \\
  & \rotatebox[origin=c]{180}{$\Lsh$} semiotic (3-way clf) & I-SEM &  \convertto{1299}{1.3K} \\
  & \rotatebox[origin=c]{180}{$\Lsh$} contextual (3-way clf) & I-CTX &  \convertto{1299}{1.3K} \\
  \newcite{vempala2019categorizing} & Twitter visual-ness (4-way clf) & T-VIS & \convertto{4471}{3.9K} \\
  \newcite{hessel2017cats} & Reddit popularity (pairwise ranking) & R-POP & \convertto{88K}{87.2K} \\
  \newcite{borth2013large} & Twitter sentiment (binary clf) & T-ST1 & \convertto{603}{.6K} \\
  \newcite{MVSA} & Twitter sentiment (binary clf) & T-ST2 & \convertto{4511}{4.0K} \\
  \bottomrule
\end{tabular}
}
  }
\caption{The tasks we consider are not specifically balanced to force the learning of cross-modal interactions.
} 
\label{tab:datasets_and_tasks}
\end{table*}

We now return to our original setting: multimodal classification tasks
that have not been specifically formulated to force cross-modal
interactions. Our goal is to explore what additional insights \emap
can add on top of standard model comparisons.

We consider a suite of 7 tasks, summarized in
Table~\ref{tab:datasets_and_tasks}. These tasks span a wide variety of
goals, sizes, and structures: some aim to classify semiotic
properties of image+text posts, e.g., examining the extent of literal
image/text overlap (I-SEM, I-CTX, T-VIS); others are post-hoc
annotated according to taxonomies of social interest (I-INT, T-ST1,
T-ST2); and one aims to directly predict community response to content
(R-POP).\footnote{
  In \autoref{sec:dataset_details}, we
  include descriptions of our reproduction efforts for each
  dataset/task, but please see the original papers for fuller
  descriptions of the data/tasks.}
In some cases, the original authors emphasize the potentially complex
interplay between image and text: \newcite{hessel2017cats} wonder if ``visual
and the linguistic interact, sometimes reinforcing and sometimes
counteracting each other's individual influence;''
\newcite{kruk2019integrating} discuss meaning multiplication,
emphasizing that ``the text+image integration requires inference that
creates a new meaning;'' and \newcite{vempala2019categorizing}
attribute performance differences between their ``naive'' additive
baseline and their interactive neural model to the notion that ``both
types of information and their interaction are important to this
task.''

\emph{Our goal is not to downplay the importance of these
\underline{datasets and tasks};} it may well be the case that conditional,
compositional interactions occur between images and text, but the
models we consider do not yet take full advantage of them (we return to this point
in \autoref{sec:discussion}).
Our goal, rather, is to provide diagnostic tools that can provide
additional clarity on the \emph{remaining} shortcomings of current models and
reporting schemes.

\subsection*{Additive and Unimodal Models}
\label{sec:sec_with_models}

The additive model we consider is the linear model from \autoref{eq:linear_model} trained over the
concatenation of image and text representations.\footnote{For all linear models in
  this work, we select optimal hyperparameters according to grid
  search, optimizing validation model performance for each
  cross-validation split separately. We optimize: regularization type
  (L1, L2, vs. L1/L2), regularization strength
  (10**(-7,-6,-5,-4,-3,-2,-1,0,1.0)), and loss type (logistic
  vs. squared hinge). We train models using \texttt{lightning}
  \cite{lightning_2016}. } To represent images, we extract a feature
vector from
a pretrained EfficientNet B4\footnote{For reproducibility, we used ResNet-18
  features for the \newcite{kruk2019integrating} datasets;
   more detail in \autoref{sec:reproducibility}.}
\cite{tan2019efficientnet}. To represent text, we extract RoBERTa
\cite{liu2019roberta} token features, and mean pool.\footnote{Feature
  extraction approaches are known to produce competitive results
  relative to full fine-tuning \cite[\S 5.3]{devlin2018bert}; in some
  cases, mean pooling has been found to be similarly competitive
  relative to LSTM pooling \cite{hessel-lee-2019-somethings}.} Our
single-modal baselines are linear models fit over EfficientNet/RoBERTa
features directly.

\begin{table*}
  \centering
  \begin{tabular}{lccccccc}
  \toprule
  & I-INT & I-SEM & I-CTX & T-VIS & R-POP & T-ST1 & T-ST2 \\
  \midrule
  Metric & AUC & AUC & AUC & Weighted F1 & ACC & AUC & ACC \\
  Cross-val Setup & 5-fold & 5-fold & 5-fold & 10-fold & 15-fold & 5-fold & 5-fold  \\
  Constant Pred. & 50.0 & 50.0 & 50.0 & 17.2 & 50.0 & 50.0 & 66.2 \\
  Prev. SOTA & 85.3 & 69.1 & 78.8 & 44 & 62.7 & N/A & 70.5 \\
  \midrule
  Our image-only & 73.6 & 56.5 & 61.0 & 47.2 & 59.1 & 73.3 & 67.2 \\
  Our text-only & 89.9 & 71.8 & \textbf{81.2} & 37.6 & 61.1 & 69.0 & 73.1 \\
  \midrule
  Neural Network (I)
  & {\color{lightgray}90.4} & {\color{lightgray}69.2} & {\color{lightgray}78.5} & {\color{lightgray}51.1} & {\color{lightgray}63.5} & {\color{lightgray}71.1} & {\color{lightgray}79.9} \\
  Polykernel SVM (I)
  & \tikzmark{intemapA}\textbf{91.3}\tikzmark{intintA} & \tikzmark{sememapA}\textbf{74.4}\tikzmark{semintA} & \tikzmark{ctxemapA}\textbf{81.5}\tikzmark{ctxintA} & {\color{lightgray}50.8} & -- & {\color{lightgray}72.1} & \tikzmark{tst2emapA}\textbf{80.9}\tikzmark{tst2intA} \\
  FT LXMERT (I) & {\color{lightgray}83.0} & {\color{lightgray}68.5} & {\color{lightgray}76.3} & {\color{lightgray}\textbf{53.0}} & {\color{lightgray}63.0} & {\color{lightgray}66.4} & {\color{lightgray}78.6} \\
  \rotatebox[origin=c]{180}{$\Lsh$} + Linear Logits (I) & {\color{lightgray}89.9} & {\color{lightgray}73.0} & {\color{lightgray}80.7} & \tikzmark{tvisemapA}\textbf{53.4}\tikzmark{tvisintA} & \tikzmark{rpopemapA}\textbf{64.1}\tikzmark{rpopintA} & \tikzmark{tst1emapA}\textbf{75.5}\tikzmark{tst1intA} & {\color{lightgray}80.3} \\
  \midrule
  Linear Model (A)
  & 90.4 & 72.8 & 80.9 & 51.3 & 63.7 & \textbf{75.6} & 76.1 \\
  Our Best Interactive (I)
  & \textbf{91.3}\tikzmark{intintB} & \textbf{74.4}\tikzmark{semintB} & \textbf{81.5}\tikzmark{ctxintB} & \textbf{53.4}\tikzmark{tvisintB} & \emph{\textbf{64.2}}\tikzmark{rpopintB} & \textbf{75.5}\tikzmark{tst1intB} & \textbf{80.9}\tikzmark{tst2intB} \\

  {\color{emapcolor} \rotatebox[origin=c]{180}{$\Lsh$} + \emap (A)}
  &\tikzmark{intemapB}{\color{emapcolor}\textbf{91.1}} & \tikzmark{sememapB}{\color{emapcolor}\textbf{74.2}} & \tikzmark{ctxemapB}{\color{emapcolor}\textbf{81.3}} & \tikzmark{tvisemapB}{\color{emapcolor}51.0} & \tikzmark{rpopemapB}{\color{emapcolor}\emph{\textbf{64.1}}}& \tikzmark{tst1emapB}{\color{emapcolor}\textbf{75.9}} & \tikzmark{tst2emapB}{\color{emapcolor}\textbf{80.7}}\\
  \bottomrule
\end{tabular}

\begin{tikzpicture}[overlay, remember picture, shorten >=5pt, shorten <=4pt, bend angle=25, line width=.4mm]
  \draw [->] ([yshift=.13cm]{pic cs:intintA}) [bend left] to ([yshift=.1cm]{pic cs:intintB});
  \draw [emapcolor,->] ([yshift=.13cm]{pic cs:intemapA}) [bend right] to ([yshift=.1cm]{pic cs:intemapB});
  
  \draw [->] ([yshift=.13cm]{pic cs:semintA}) [bend left] to ([yshift=.1cm]{pic cs:semintB});
  \draw [emapcolor,->] ([yshift=.13cm]{pic cs:sememapA}) [bend right] to ([yshift=.1cm]{pic cs:sememapB});

  \draw [->] ([yshift=.13cm]{pic cs:ctxintA}) [bend left] to ([yshift=.1cm]{pic cs:ctxintB});
  \draw [emapcolor,->] ([yshift=.13cm]{pic cs:ctxemapA}) [bend right] to ([yshift=.1cm]{pic cs:ctxemapB});

  \draw [->] ([yshift=.13cm]{pic cs:tvisintA}) [bend left] to ([yshift=.1cm]{pic cs:tvisintB});
  \draw [emapcolor,->] ([yshift=.13cm]{pic cs:tvisemapA}) [bend right] to ([yshift=.1cm]{pic cs:tvisemapB});

  \draw [->] ([yshift=.13cm]{pic cs:rpopintA}) [bend left] to ([yshift=.1cm]{pic cs:rpopintB});
  \draw [emapcolor,->] ([yshift=.13cm]{pic cs:rpopemapA}) [bend right] to ([yshift=.1cm]{pic cs:rpopemapB});

  \draw [->] ([yshift=.13cm]{pic cs:tst1intA}) [bend left] to ([yshift=.1cm]{pic cs:tst1intB});
  \draw [emapcolor,->] ([yshift=.13cm]{pic cs:tst1emapA}) [bend right] to ([yshift=.1cm]{pic cs:tst1emapB});

  \draw [->] ([yshift=.13cm]{pic cs:tst2intA}) [bend left] to ([yshift=.1cm]{pic cs:tst2intB});
  \draw [emapcolor,->] ([yshift=.13cm]{pic cs:tst2emapA}) [bend right] to ([yshift=.1cm]{pic cs:tst2emapB});
\end{tikzpicture}

  \caption{Prediction results for 7 multimodal classification
    tasks. First block: the evaluation metric, setup, constant
    prediction performance, and previous state-of-the-art results (we
    outperform these baselines mostly because we use RoBERTa). Second
    block: the performance of our image only/text only linear models.
    Third block: the predictive performance of our
    \textbf{(I)}nteractive models. Fourth block: comparison of the performance of
    the best \textbf{(I)}nteractive model to the \textbf{(A)}dditive
    linear baseline. Crucially, we also report the
    {\color{emapcolor}\emap of the best interactive model,} which
    reveals whether or not the performance gains of the
    \textbf{(I)}nteractive model are due to modeling cross-modal
    interactions, or not. Italics=computed using 15 fold
    cross-validation over each cross-validation split (see
    footnote~\ref{fn:prohibitively_large}). Bolded values are within
    half a point of the best model.}
  \label{tab:pred_results}
\end{table*}

\subsection*{Interactive Models}

\mparagraph{Kernel SVM} We train an SVM with a polynomial kernel using
RoBERTa text features and EfficientNet-B4 image features as
input. A polynomial kernel endows the model with capacity to model
multiplicative interactions between features.\footnote{We again use
  grid search to optimize: polynomial kernel degree (2 vs.~3),
  regularization strength (10**(-5,-4,-3,-2,-1,0)), and gamma (1, 10,
  100).}

\mparagraph{Neural Network} We train a feed-forward neural network
using the RoBERTa/EfficientNet-B4 features as input. Following
\newcite{chen2016enhanced}, we first project text and image features
via an affine transform layer to representations $t$ and $v$,
respectively, of the same dimension. Then, we extract new features,
feeding the concatenated feature vector $[t; v; v-t; v \odot t]$ to a
multi-layer, feed-forward network.\footnote{Parameters are optimized
  with the Adam optimizer \cite{kingma2014adam}. We decay the learning
  rate when validation loss plateaus. The hyperparameters optimized in
  grid search are: number of layers (2, 3, 4), initial learning rate
  (.01, .001, .0001), activation function (relu vs.~gelu), hidden
  dimension (128, 256), and batch norm (use vs.~don't).}

\mparagraph{Fine-tuning a Pretrained Transformer} We fine-tuned LXMERT
\cite{tan2019lxmert} for our tasks. LXMERT represents images using 36
predicted bounding boxes, each of which is assigned a feature vector
by a Faster-RCNN model \cite{ren2015faster,anderson2018bottom}. This
model uses ResNet-101 \cite{he2016deep} as a backbone and is
pretrained on Visual Genome \cite{krishna2017visual}. Bounding box
features are fed through LXMERT's 5-layer Transformer
\cite{vaswani2017attention} model. Text is processed by a 9-layer
Transformer. Finally, a 5-layer, cross-modal transformer processes
the outputs of these unimodal encoders jointly, allowing for feature
interactions to be learned. LXMERT's parameters are pre-trained on
several cross-modal reasoning tasks, e.g., masked image
region/language token prediction, cross-modal matching, and visual
question answering. LXMERT achieves high performance on balanced tasks
like VQA 2.0: thus, we \emph{know} this model can learn compositional,
cross-modal interactions in some settings.\footnote{We follow the
  original authors' fine-tuning recommendations, but also optimize the
  learning rate according to validation set performance for each
  cross-validation split/task separately between 1e-6, 5e-6, 1e-5,
  5e-5, and 1e-4.}

\mparagraphnp{LXMERT + logits:} We also trained versions of LXMERT
where fixed logits from a pretrained linear model (described above)
are fed in to the final classifier layer. In this case, LXMERT is only
tasked with learning the residual between the strong additive model
and the labels. Our intuition was that this augmentation might enable
the model to focus more closely on learning interactive, rather than
additive, structure in the fine-tuning process.

\subsection{Results}
\label{sec:sec_where_emap_rocks}
Our main prediction results are summarized in
\autoref{tab:pred_results}. For all tasks, the performance of our
baseline additive linear model is strong, but we are usually able to
find an interactive model that outperforms this linear baseline, e.g.,
in the case of T-ST2, a polynomial kernel SVM outperforms the linear
model by 4 accuracy points. This observation alone \emph{seems} to
provide evidence that models are taking advantage of some cross-modal
interactions for performance gains. Previous analyses might conclude
here, arguing that cross-modal interactions are being utilized by the
model meaningfully. \emph{But is this necessarily the case?}

We utilize \emap as an additional model diagnostic by projecting the
predictions of our best-performing interactive models. Surprisingly,
for I-INT, I-SEM, I-CTX, T-ST1, T-ST2, and R-POP, \emap results in
\emph{essentially no performance degradation.} Thus, even (current)
expressive interactive models are usually unable to leverage
cross-modal feature interactions to improve performance. This
observation would be obfuscated without the use of additive
projections, even though we compared to a strong linear baseline that
achieves state-of-the-art performance for each dataset. \emph{This
  emphasizes the importance of not only comparing to additive/linear
  baselines, but also to the \emap of the best performing model.}

In total, for these experiments, we observe a single case, LXMERT +
Linear Logits trained on T-VIS, wherein modeling cross-modal feature
interactions appears to result in noticeable performance increases ---
here, the \emaped model is 2.4 F1 points worse.

\mparagraphnp{How much does \emap change a model's predictions?} For
these datasets, a model and its \emap usually make very similar
predictions. For all datasets except T-VIS (where \emap degrades
performance) the best model and its \emap agree on the most likely
label in more than 95\% of cases on average. For comparison,
retraining the full models with different random seeds results in a
96\% agreement on average. (Full results are in \autoref{sec:just_regularizers}.)

\section{Implication FAQs}
\label{sec:faq}

\mparagraphnp{Q1: What is your recommended experimental workflow for future
  work examining multimodal classification tasks?}

\noindent \textbf{A:} With respect to automated classifiers, we recommend reporting the performance of:
\begin{enumerate}[leftmargin=*,topsep=3pt,itemsep=-.5ex,partopsep=0ex,parsep=1ex]
\item A constant and/or random predictor

  \vspace{-.1cm}
  {\raggedleft \emph{to provide perspective on the interplay between the label distribution and the evaluation metrics.}\par}
  
\item As-strong-as-possible single-modal models $f(t,v) = g(t)$ and $f(t,v) = h(v)$

  \vspace{-.1cm}
  {\raggedleft \emph{to understand how well the task can be addressed unimodally with present techniques.}\par}
    
\item An as-strong-as-possible multimodally additive model $f(t, v) = f_T(t) + f_V(v)$

  \vspace{-.1cm}
  {\raggedleft \emph{to understand multimodal model performance without access to sophisticated cross-modal reasoning capacity.} \par}
  
\item An as-strong-as-possible multimodally interactive model, e.g., LXMERT,

  \vspace{-.1cm}
  {\raggedleft \emph{to push predictive performance as far as possible.}\par}
  
\item The \emap of the strongest multimodally interactive model.

  \vspace{-.1cm}
  {\raggedleft \emph{to determine whether or not the best interactive model is truly using cross-modal interactions to improve predictive performance.}\par}
  
\end{enumerate}

\noindent We hope this workflow can be extended with additional, newly
developed model diagnostics going forward.

\mparagraphnp{Q2: Should papers proposing new tasks be rejected if
  image-text interactions aren't shown to be useful?}

\noindent \textbf{A:} \emph{Not necessarily.} The 
value
of a newly proposed task should not depend solely on how well current
models perform on it. 
Other valid ways
to demonstrate dataset/task efficacy: human experiments, careful
dataset design inspired by prior work, or real-world use-cases.

\mparagraphnp{Q3: Can \emap tell us anything fundamental about the type of
  reasoning required to address different tasks themselves?}

\noindent \textbf{A:} \emph{Unfortunately, no more than model
  comparisons can (at least for real datasets).} \emap is a tool that,
like model comparison, provides insights about how specific, fixed
models perform on specific, fixed datasets. In FAQ 5, we attempt to
bridge this gap in a toy setting where we are able to fully enumerate
the sample space.

\mparagraphnp{Q4: Can \emap tell us anything about individual instances?}

\noindent \textbf{A:} \emph{Yes; but with caveats.} While a model's
behavior on individual cases is difficult to draw conclusions from,
\emap can be used to identify single instances within evaluation
datasets for which the model is performing enough non-additive
computation to change its ultimate prediction, i.e., for a given $(t,
v)$, it's easy to compare $f(t, v)$ to $\emap(f(t,v))$: these
correspond to the inputs and outputs respectively of
\autoref{alg:additive_projection}. An example instance-level
qualitative evaluation of T-VIS is given in
\autoref{sec:qualitative_tvis}.

\mparagraphnp{Q5: 
Do the strong \emap results imply that most functions of interest are
actually multimodally additive?} 
\label{sec:discussion}

\noindent \textbf{Short A:} \emph{No.}

\noindent \textbf{Long A:} A skeptical reader
might argue that, while multimodally-additive models cannot account
for cross-modal feature interactions, such feature interactions may
not be required for cross-modal reasoning. While we authors are not
aware of an agreed-upon definition,\footnote{
  \newcite{suhr2018corpus}, for example, do not 
  define
  ``visual reasoning with natural language,'' but do argue that
  some tasks offer a promising avenue to study ``reasoning that
  requires composition'' via visual-textual grounding.} we will assume
that ``cross-modal reasoning tasks'' are those that challenge models
to compute (potentially arbitrary) logical functions of multimodal
inputs. Under this assumption, we show that additive models cannot fit
(nor well-approximate) most non-trivial logical functions.

Consider the case of multimodal boolean target functions $f(t, v) \in
\{0, 1\}$, and assume our image/text input feature vectors each
consist of $n$ binary features, i.e., $t, v = \langle t_1, \ldots t_n
\rangle, \langle v_1, \ldots v_n \rangle$ where $t_i, v_i \in \{0,
1\}$. Our goal will be to measure how well multimodally additive
models can fit arbitrary logical functions $f$.  To simplify our
analysis in this idealized case: we assume 1) access to a training set
consisting of all $2^{2n}$ input vector pairs, and only measure
training accuracy (vs.  the harder task of generalization) and 2) ``perfect'' unimodal representations in the sense that $t, v$
contain all of the information required to compute $f(t, v)$ (this is
not the case for, e.g., CNN/RoBERTa features for actual tasks).

For very small cross-modal reasoning tasks, additive models can 
suffice. At $n=1$ , there are 16 possible binary functions
$f(t_1, v_1)$, and 14/16 can be perfectly fit by a function of the
form $f_T(t_1) + f_V(v_1)$ (the exceptions being {\sf XOR} and {\sf
  XNOR}). For $n=2$, non-trivial functions are still often multimodally
additively representable; an arguably surprising example is this one:
\resizebox{0.95\hsize}{!}{%
${\displaystyle (t2 \land \lnot v2) \lor (t1 \land t2 \land v1) \lor
    (\lnot t1 \land \lnot v1 \land \lnot v2)}.$
}
But for cross-modal reasoning tasks with more variables,\footnote{As a
  lower bound on the number of variables required in a more realistic
  case, consider a cross-modal reasoning task where questions are
  posed regarding the counts of, say, 1000 object types in
  images. It's likely that a separate image variable representing the
  count of each possible type of object would be required. While an
  oversimplification, for most cross-modal reasoning tasks, $2n=6$
  variables is still relatively small.} even in this toy setting,
multimodally-additive models ultimately fall short: for the $n=3$
case, we sampled over 5 million random logical circuits and none were
perfectly fit by the additive models. To better understand how well
additive models can \emph{approximate} various logical functions, we
fit two types of them for the $n>2$ case: 1) the \emap of the input
function $f$ directly, and 2) AdaBoost \citep{freund1995desicion} with the weak learners
restricted to unimodal models.\footnote{AdaBoost is chosen because it
  has strong theoretical guarantees to fit to training data: 
  see \autoref{sec:adaboost}.} For a reference interactive model, we
employ AdaBoost without the additivity restriction.
\begin{wrapfigure}{r}{0.4\linewidth}
  \begin{center}
    \includegraphics[width=0.95\linewidth]{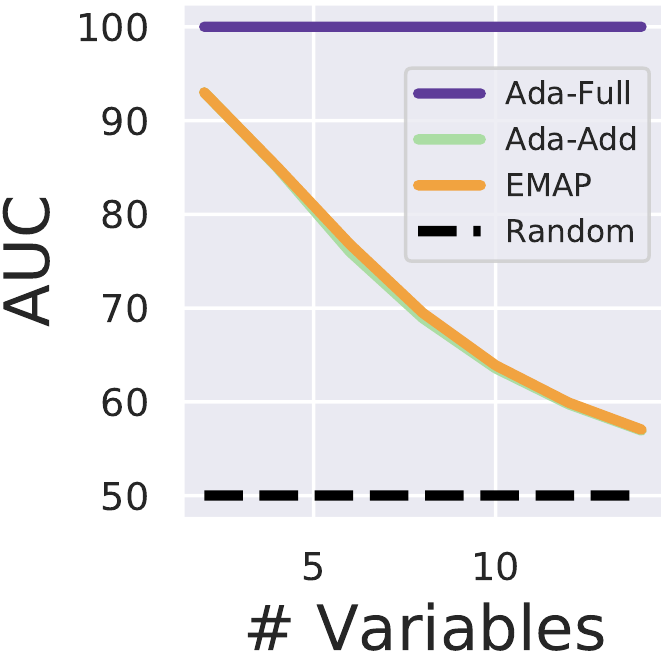}
  \end{center}
  \caption{Perf. of additive models on fitting logical circuits.}
  \label{fig:logic_curve}
\end{wrapfigure}

\autoref{fig:logic_curve} plots the AUC of the 3 models on the
training set; we sample 10K random logical circuits for each problem
size.  As the number of variables increases, the performance of the
additive models quickly decreases (though full AdaBoost gets 100\%
accuracy in all cases). While these experiments are only for a toy
case, they show that \emaps, and additive models generally, have very
limited capacity to compute or approximate logical functions of
multimodal variables.

\section{Conclusion and Future Work}

The last question on our FAQ list in \autoref{sec:faq} leaves us
with the
following conundrum: 1) Additive models are incapable of most
cross-modal reasoning; but 2) for most of the unbalanced tasks we
consider, \emap finds an additive approximation that makes nearly
identical predictions to the full, interactive model. We postulate the
following potential explanations, pointing towards future work:

\begin{itemize}[leftmargin=*,topsep=0pt,itemsep=-1ex,partopsep=1ex,parsep=1ex]
\item Hypothesis 1: {These unbalanced tasks don't require complex
  cross-modal reasoning.} This purported conclusion cannot account for
  gaps between human and machine performance: if an additive model
  underperforms relative to human judgment, the gap could plausibly be
  explained by cross-modal feature interactions. But even in cases
  where an additive model matches or exceeds human performance on a
  fixed dataset, additive models may still be insufficient. The mere
  fact that unimodal and additive models \emph{can} often be disarmed
  by adding valid (but carefully selected) instances post hoc (as in,
  e.g., \newcite{kiela2020hateful}) suggests that their inductive bias
  can simultaneously be sufficient for train/test generalization, but
  also fail to capture the spirit of the task. Future work would be
  well-suited to explore 1) %
  methods for better understanding which datasets (and individual
  instances) can be rebalanced and which cannot; and 2) the
  non-trivial task of estimating additive \emph{human} baselines to
  compare against.

\item Hypothesis 2: {Modeling feature interactions can be data-hungry.}
  \newcite{Jayakumar2020Multiplicative} show that feed-forward neural
  networks can require a very high number of training examples to
  learn feature interactions. So, we may need models with different
  inductive biases and/or much more training data. Notably, the
  feature interactions learned even in balanced cases are often not
  interpretable \cite{subramaniananalyzing}.

\item Hypothesis 3: {Paradoxically, unimodal models may be too weak.} Without
  expressive enough single-modal processing methods, opportunities for
  learning cross-modal interaction patterns may not be present during
  training. So, improvements in unimodal modeling could feasibly
  improve feature interaction learning.
  
\end{itemize}

\mparagraphnp{Concluding thoughts.}  Our hope is that future work on
multimodal classification tasks report not only the predictive
performance of their best model + baselines, but also the \emap of
that model. \emap (and related algorithms) has practical implications
beyond image+text classification: there are straightforward extensions
to non-visual/non-textual modalities, to classifiers using more than
2 modalities as input, and to single-modal cases where one wants to
check for feature interactions between two groups of features, e.g.,
premise/hypothesis in NLI.

\section*{Acknowledgments} In addition to the anonymous reviewers and meta-reviewer, the authors would like to
thank Ajay Divakaran, Giles Hooker, Jon Kleinberg, Julia Kruk, Tianze
Shi, Ana Smith, Alane Suhr, and Gregory Yauney for the helpful
discussions and feedback. We thank the NVidia Corporation for
providing some of the GPUs used in this study. 
JH performed this work while at Cornell University.
Partial support for this work
was kindly provided by a grant from The Zillow Corporation, and a
Google Focused Research Award. 
Any opinions, findings, conclusions, or recommendations
expressed here are those of the authors and do not 
necessarily reflect the view of the sponsors.

\putbib[refs]
\end{bibunit}

\clearpage
\begin{bibunit}[acl_natbib]
\appendix
\section{Theoretical Analysis of \autoref{alg:additive_projection}}
\label{sec:theoretical}

We assume we are given a (usually evaluation) dataset $D = \{ \langle
t_i, v_i \rangle \}_{i=1}^n$ and a trained model $f$ that maps
$\langle t_i, v_i \rangle$ pairs to a $d$-dimensional vector of
scores. We seek a multimodally-additive function $\hat f$ that matches
the values of $f$ on any $\langle t_i, v_j \rangle$ for which there
exist $v', t'$ such that $\langle t_i, v' \rangle \in D$ and $\langle
t', v_j \rangle \in D$; that is, $\langle t_i, v_j \rangle$ represents
any text-image pair we could construct if we decoupled the existing
pairs in $D$.\footnote{Note that multimodally-additive models do not
  rely on particular $t_i, v_j$ couplings, as this family of functions
  decomposes as $f(t_i, v_j) = f_t(t_i) + f_v(v_j)$; thus, a
  multimodally-additive $\hat f$ \emph{should} be able to fit any
  image-text pair we could construct from $D$, not just the
  image-text pairs we observe.}

For simplicity, first assume that $d=1$ (we handle the $d>1$ case
later). Since $\hat f$ is multimodally-additive by assumption,
$\exists \, \hat f_t, \hat f_v$ such that $\hat f(t_i, v_j) = \hat
f_t(t_i) + \hat f_v(v_j)$. Our goal is to find the ``best'' $2n$
values $\hat f_t(t_i), \hat f_v(v_j)$, or --- writing $f_{ij}$ for
$f(t_i, v_j)$, $\tparam$ for $\hat f_t(t_i)$, and $\vparam$ for $\hat
f_v(v_j)$ for notational convenience --- to find $\tparam, \vparam$
minimizing
\begin{equation}
  \mathcal{L} = \frac{1}{2} \sum_i \sum_j (f_{ij} - \tparam - \vparam)^2
  \label{eq:loss_appendix}
\end{equation}

\begin{claim}
  $\mathcal{L}$ is convex.
\end{claim}

\begin{proof}
  The first-order partial derivatives are:
  $$\frac{\partial \mathcal{L}}{\partial \tparam} = n \cdot \tparam + \sum_j (\vparam - f_{ij}) $$
  $$\frac{\partial \mathcal{L}}{\partial \vparam} = n \cdot \vparam + \sum_i (\tparam - f_{ij}) $$
  and the Hessian $\mathcal{H}$ is

  $$
  \mathcal{H} =
  \begin{bmatrix}
    n I & \mathbf{1}  \\
    \mathbf{1} & n I
  \end{bmatrix} , \;\; I, \mathbf{1} \in \mathbb{R}^{n \times n}
  $$
  It suffices to show that $\mathcal{H}$ is positive semi-definite, i.e.,
  for any $z \in \mathbb{R}^{2n}$, $z^T \mathcal{H} z \geq 0$. Indeed,

  \begin{align*}
  z^T \mathcal{H} z &= z^T \begin{bmatrix}
    n z_1 + \sum\limits_{j=n+1}^{2n} z_j  \\
    n z_2 + \sum\limits_{j=n+1}^{2n} z_j  \\
    \vdots \\
    n z_n + \sum\limits_{j=n+1}^{2n} z_j  \\
    n z_{n+1} + \sum\limits_{i=1}^n z_i  \\
    n z_{n+2} + \sum\limits_{i=1}^n z_i  \\
    \vdots \\
    n z_{2n} + \sum\limits_{i=1}^n z_i  \\
  \end{bmatrix} \\
  &= \sum\limits_{i=1}^n \left( nz_i^2 + \sum\limits_{j=n+1}^{2n} z_iz_j \right) \\
  &+ \sum\limits_{j=n+1}^{2n} \left( n z_j^2 + \sum\limits_{i=1}^n z_jz_i \right) \\
  &= \sum\limits_{i=1}^n \sum\limits_{j=n+1}^{2n} \left(z_i^2 + 2z_iz_j + z_j^2 \right) \\
  &= \sum\limits_{i=1}^n \sum\limits_{j=n+1}^{2n} (z_i + z_j)^2 \geq 0
  \end{align*}
\end{proof}

Now, for the optimal solutions to our minimization problem, we can set
the first-order partial derivatives to 0 and solve for our $2n$
parameters $\tparam, \vparam$. These solutions will correspond to global
minima due to the convexity result we established above. It turns out
to be equivalent to find solutions to:
\begin{equation}
  \mathcal{H}  \begin{bmatrix}
  \tau_1  \\
  \vdots \\
  \tau_n \\
  \phi_1  \\
  \vdots \\
  \phi_n
  \end{bmatrix} =
  \begin{bmatrix}
  \sum\limits_{k=1}^n f_{1k}  \\
  \vdots \\
  \sum\limits_{k=1}^n f_{nk} \\
  \sum\limits_{k=1}^n f_{k1}  \\
  \vdots \\
  \sum\limits_{k=1}^n f_{kn}
  \end{bmatrix}
  \label{eq:system_of_eq}
\end{equation}

\noindent We can do this by finding one solution $s$ to the
above, and then analyzing the nullspace of $\mathcal{H}$,
which will turn out to be the 1-dimensional subspace
spanned by
$$ r=\langle \underbrace{1, 1, \ldots 1,}_{n}\, \underbrace{-1, -1,
  \ldots -1}_n \rangle$$

That $\mathcal{H}r = 0$ can be verified by direct calculation.  Then,
all solutions will have the form $s+cr$ for any $c \in
\mathbb{R}$.\footnote{Proof: Assume that $s'$ is a solution of
  \autoref{eq:system_of_eq}.  $s'-s$ will be in the
  nullspace of $\mathcal{H}$. Clearly, $s' = s + (s' - s)$, so 
  $s'$ can be written as $s + x$ for $x$ in the nullspace of
  $\mathcal{H}$.}

\begin{claim}
  \autoref{alg:additive_projection} computes a solution to
  \autoref{eq:system_of_eq} as a byproduct.
\end{claim}
\begin{proof}
  \autoref{alg:additive_projection} computes $s = \langle \tparam, \vparam \rangle$ as:
  \begin{align}
  \tparam &= \frac{1}{n} \sum_k f_{ik} - \frac{1}{n^2} \sum_i \sum_j f_{ij} \label{eq:tiopt} \\
  \vparam &= \frac{1}{n} \sum_k f_{kj} \label{eq:vjopt} 
  \end{align}
  $s$ is a solution to \autoref{eq:system_of_eq}, as can be verified by direct substitution.
\end{proof}

\begin{claim}
  The rank of $\mathcal{H}$ is $2n-1$, which implies that its
  nullspace is 1-dimensional.
\end{claim}

\begin{proof}
  Solutions to $\mathcal{H}x = \lambda x$ occur when:
  \begin{align*}
  \lambda &= n, \; 0 = \sum\limits_{i=1}^n x_i \text{ and } 0 =\sum\limits_{j=n+1}^{2n} x_j \\
  \lambda &= 2n, \; x = \mathbf{1} \\
  \lambda &= 0, \; x = r
  \end{align*}
    So, zero is an eigenvalue of
  $\mathcal{H}$ with multiplicity 1, which shows that $\mathcal{H}$'s rank is
  $2n-1$.\footnote{ We can make explicit the eigenbasis for the
    $\lambda=n$ solutions.  Let $M \in \mathbb{R}^{n \times (n-1)}$ be
    $I_{n-1}$ with an additional row of $-1$ concatenated as the $n^{th}$ row.
    The eigenbasis is given by the columns of
    $$
    \begin{bmatrix} M & \mathbf{0} \\ \mathbf{0} & M \end{bmatrix}.
    $$
  }
\end{proof}

Because the nullspace of $\mathcal{H}$ is 1-dimensional, all solutions to
\autoref{eq:system_of_eq} are given by $s + cr$ for any $c \in
\mathbb{R}$. Returning to the notation of the original problem, we see
that all optimal solutions are given by:
\begin{align}
  \tparam &= \frac{1}{n} \sum_k f_{ik} - \frac{1}{n^2} \sum_i \sum_j f_{ij} + c \label{eq:all_ti}\\
  \vparam &= \frac{1}{n} \sum_k f_{kj} - c \label{eq:all_vj}
\end{align}

\begin{claim}
  \autoref{alg:additive_projection} produces a unique solution
  for the values of $\hat f$.
\end{claim}
\begin{proof}
  We have shown that \autoref{alg:additive_projection} produces
  \emph{an} optimal solution, and have derived the parametric form of
  all optimal solutions in \autoref{eq:all_ti} and 
  \autoref{eq:all_vj}. Note that \autoref{alg:additive_projection}
  outputs $\tparam + \vparam$ (rather than $\tparam, \vparam$
  individually). This cancels out the free choice of $c$. Thus, any
  algorithm that outputs optimal $\tparam + \vparam$ will have the
  same output as \autoref{alg:additive_projection}.
\end{proof}

\mparagraph{Extension to multiple dimensions} So far, we have shown
that \autoref{alg:additive_projection} produces an optimal and
unique solution for \autoref{eq:loss}, but only in cases where
$f_{ij}, \tparam, \vparam \in \mathbb{R}$. In general, we need to show the
algorithm works for multi-dimensional outputs, too. The full loss
function includes a summation over dimension as:
\begin{equation}
  \mathcal{L} = \frac{1}{2} \sum_i \sum_j \sum_d (f_{ijd} - \tau_{id} - \phi_{jd})^2
  \label{eq:loss_full}
\end{equation}
This objective decouples entirely over dimension $d$,
i.e., the loss can be rewritten as:
\begin{align*}
  \frac{1}{2} \bigg( & \underbrace{\sum_{ij} (f_{ij1} - \tau_{i1} - \phi_{j1})^2}_{\mathcal{L}_1} + \\
  & \underbrace{\sum_{ij} (f_{ij2} - \tau_{i2} - \phi_{j2})^2}_{\mathcal{L}_2} + \ldots \\
  & \underbrace{\sum_{ij} (f_{ijd} - \tau_{id} - \phi_{jd})^2}_{\mathcal{L}_d} \bigg)
\end{align*}
Furthermore, notice that the parameters in $\mathcal{L}_i$ are
disjoint from the parameters in $\mathcal{L}_j$ if $i \neq j$. Thus,
to minimize the multidimensional objective in
\autoref{eq:loss_full}, it suffices to minimize the objective for
each $\mathcal{L}_i$ independently, and recombine the solutions. This
is precisely what \autoref{alg:additive_projection} does.

\section{Dataset Details and  Reproducibility Efforts}
\label{sec:dataset_details}

\subsection{I-INT, I-SEM, I-CTX}
This data is available from
\url{https://github.com/karansikka1/documentIntent_emnlp19}. We use
the same 5 random splits provided by the authors for evaluation. The
authors provide ResNet18 features, which we use for our non-LXMERT
experiments instead of EfficientNet-B4 features. After contacting the
authors, they extracted bottom-up-top-down FasterRCNN features for us,
so we were able to compare to LXMERT. State of the art performance
numbers are derived from the above github repo;  these differ slightly
from the values reported in the original paper because the github
versions are computed without image data augmentation.

\subsection{T-VIS}
This data is available from
\url{https://github.com/danielpreotiuc/text-image-relationship/}. The
raw images are not available, so we queried the Twitter API for them.
The corpus has 4472 tweets in it initially, but we were only able to
re-collect 3905 tweets (87\%) when we re-queried the API. Tweets can
be missing for a variety of reasons, e.g., the tweet being permanently
deleted, or the account's owner making the their account private at
the time of the API request. A handful of tweets were available, but
were missing images when we tried to re-collect them. This can happen
when the image is a link to an external page, and the image is deleted
from the external page.

\begin{table*}
  \resizebox{\textwidth}{!}{
\begin{tabular}{lccccccc}
  \toprule
  & I-INT & I-SEM & I-CTX & T-VIS & R-POP & T-ST1 & T-ST2 \\
  \midrule
  Metric & AUC & AUC & AUC & Weighted F1 & ACC & AUC & ACC \\
  Num Classes & 7 & 3 & 3 & 4 & 2 & 2 & 2 \\
  Setup & 5-fold & 5-fold & 5-fold & 10-fold & 15-fold & 5-fold & 5-fold  \\
  Best Interactive & Poly & Poly & Poly & LXMERT & LXMERT & LXMERT & Poly \\
  \midrule
  Original Perf. & 91.3 & 74.4 & 81.5 & 53.4 & \emph{64.2} & 75.5 & 80.9 \\
  Original \emap & 91.1 & 74.2 & 81.3 & 51.0 & \emph{64.1}& 75.9 & 80.7 \\
  DiffSeed Perf. & 91.3 & 74.5 & 81.4 & 53.2 & \emph{64.1}& 75.3 & 81.3 \\
  \midrule
  Match Orig + \emap & 95.6 & 95.9 & 97.4 & 85.5 & \emph{96.3} & 98.0 & 96.7 \\
  Match Orig + DiffSeed & 99.9 & 99.1 & \multicolumn{1}{r}{100.0} & 75.5 & \emph{87.6} & 92.4  & 97.9 \\
  \% Inst. Orig. Better & 51.2 & 52.0 & 51.5 & 55.2 & \emph{51.2} & 5/12 cases& 53.0 \\
   & & & & & & {\tiny ($\approx 6/12=50\%$)} & \\
  \bottomrule
\end{tabular}
}

  \caption{Consistency results. The first block provides details about
    the task and the model that performed best on it. The second block
    gives the performance (italicized results represent
    cross-validation \emap computation results; see
    footnote~\ref{fn:prohibitively_large}). The third block gives the
    percent of time the original model's prediction is the same as for
    \emap, and, for comparison, the percent of time the original
    model's predictions match the identical model trained with a
    different random seed: in all cases except for T-VIS, the original
    model and the \emap make the same prediction in more than 95\% of
    cases. The final row gives the percent of instances (among
    instances for which the original model and the \emap disagree)
    that the original model is correct. Except for T-VIS, when the
    \emap and the original model disagree, each is right around half
    the time.}
  \label{tab:consistency_results}
\end{table*}

\subsection{R-POP}
This data is available from
\url{http://www.cs.cornell.edu/~jhessel/cats/cats.html}.  We just use
the pics subreddit data. We attempted to rescrape the pics images from
the imgur urls. We were able to re-collect 87215/88686 of the images
(98\%). Images can be missing if they have been, e.g., deleted from
imgur. We removed any pairs with missing images from the ranking task;
we trained on 42864/44343 (97\%) of the original pairs. The data
is distributed with training/test splits. From the training set for
each split, we reserve 3K pairs for validation. The state of the art
performance numbers are taken from the original releasing work.

\subsection{T-ST1}
This data is available from
\url{http://www.ee.columbia.edu/ln/dvmm/vso/download/twitter_dataset.html}
and consists of 603 tweets (470 positive, 133 negative). The authors
distribute data with 5 folds pre-specified for cross-validation
performance reporting. However, we note that the original paper's
best model achieves 72\% accuracy in this setting, but a constant
prediction baseline achieves higher performance:  470/(470+133) =
78\%. Note that the constant prediction baseline likely performs worse
according to metrics other than accuracy, but only accuracy is
reported. We attempted to contact the authors of this study but did
not receive a reply. We also searched for additional baselines for
this dataset, but were unable to find additional work that uses this
dataset in the same fashion. Thus, given the small size of the
dataset, lack of reliable measure of SOTA performance, and label
imbalance, we decided to report ROC AUC prediction performance.

\subsection{T-ST2}

This data is available from
\url{https://www.mcrlab.net/research/mvsa-sentiment-analysis-on-multi-view-social-data/}.
We use the MVSA-Single dataset because human annotators examine both
the text and image simultaneously; we chose not to use MVSA-Multiple
because human annotators do not see the tweet's image and text at the
same time. However, the dataset download link only comes with 4870
labels, instead of the 5129 described in the original paper. We
contacted the authors of the original work about the missing data, but
did not receive a reply.

We follow the preprocessing steps detailed in
\newcite{xu2017multisentinet} to derive a training dataset.  After
preprocessing, we are left with 4041 data points, whereas prior work
compares with 4511 points after preprocessing. The preprocessing
consists of removing points that are (positive, negative), (negative,
positive), or (neutral, neutral), which we believe matches the
description of the preprocessing in that work. We contacted the
authors for details, but did not receive a reply. The state-of-the-art
performance number for this dataset is from \newcite{xu2018co}.

\section{Are \emaps just regularizers?}
\label{sec:just_regularizers}
One reason why \emaps may often offer strong performance is by acting as
a regularizer: projecting to a less expressive hypothesis space may
reduce variance/overfitting. But in many cases, the original model and
its \emap achieve similar predictive accuracy. This suggests two
explanations: \emph{either} the \emaped version makes significantly
different predictions with respect the original model (e.g., because it is
better regularized), but it happens that those differing predictions
``cancel out'' in terms of final prediction accuracy;  \emph{or}, the original, unprojected functions are quite close
to additive, anyway, and the \emap doesn't change the predictions all
that much.

We differentiate between these two hypotheses by measuring the percent
of instances for which the \emap makes a different classification
prediction than the full model. \autoref{tab:consistency_results}
gives the results: in all cases except T-VIS, the match between \emap
and the original model is above 95\%. For reference, we retrained the
best performing models with different random seeds, and measured the
performance difference under this change.

When \emap changes the predictions of the original model, does the
projection generally change to a more accurate label, or a less
accurate one? We isolate the instances where a label swap occurs, and
quantify this using: (num orig better) / (num orig better + num proj
better). In most cases, the effect of projecting improves and degrades
performance roughly equally, at least according to this metric. For
T-VIS, however, the original model is better in over 55\% of cases:
this is also reflected in the corresponding F-scores.

\section{Logic Experiment Details}
\label{sec:adaboost}

In \S\ref{sec:discussion}, we describe experiments using AdaBoost. We
chose AdaBoost \cite{freund1995desicion} to fit to the training set
because of its strong convergence guarantees. In short: if AdaBoost
can find a weak learner at each iteration (that is: if it can find a
candidate classifier with above-random performance) it will be able to
fit the training set.  A more formal statement of AdaBoost's
properties can be found in the original work.

The AdaBoost classifiers we consider use decision trees with max depth
of 15 as base estimators. We choose a relatively large depth because
we are not concerned with overfitting: we are just measuring training
fit. The additive version of AdaBoost we consider is identical to the
full AdaBoost classifier, except, at each iteration, either the image
features or the text features individually are considered.

\section{Additional Reproducibility Info}
\label{sec:reproducibility}
\begin{figure*}[ht]
  \hfill
  \includegraphics[width=.24\textwidth]{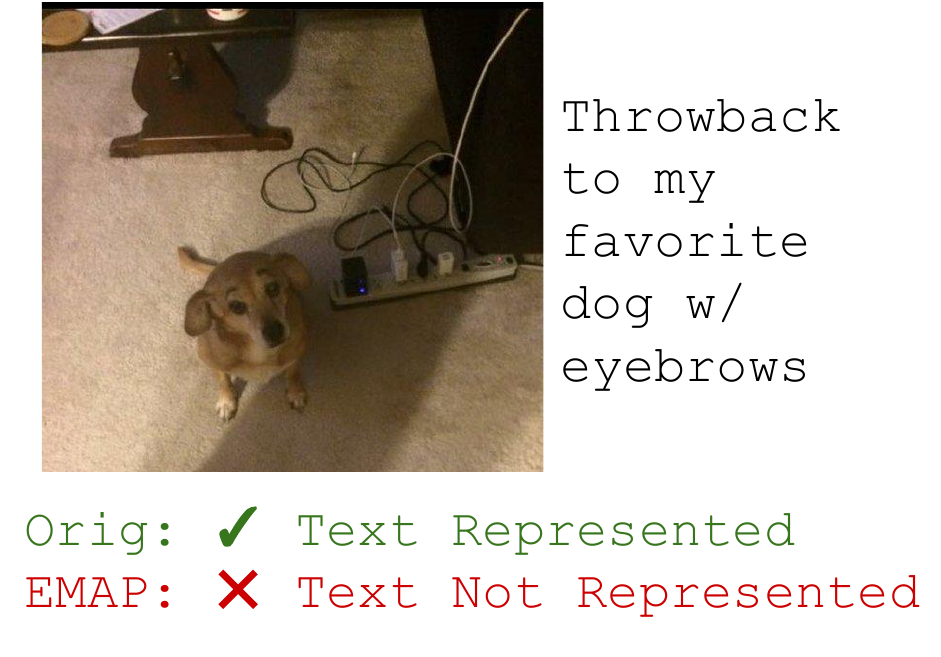}
  \hfill
  \includegraphics[width=.24\textwidth]{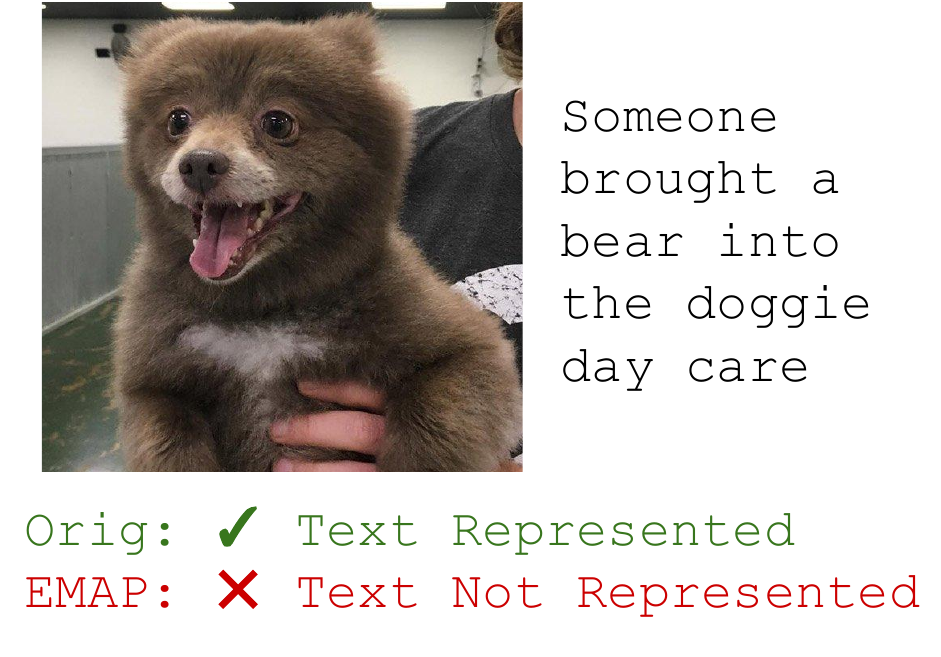}
  \hfill
  \includegraphics[width=.24\textwidth]{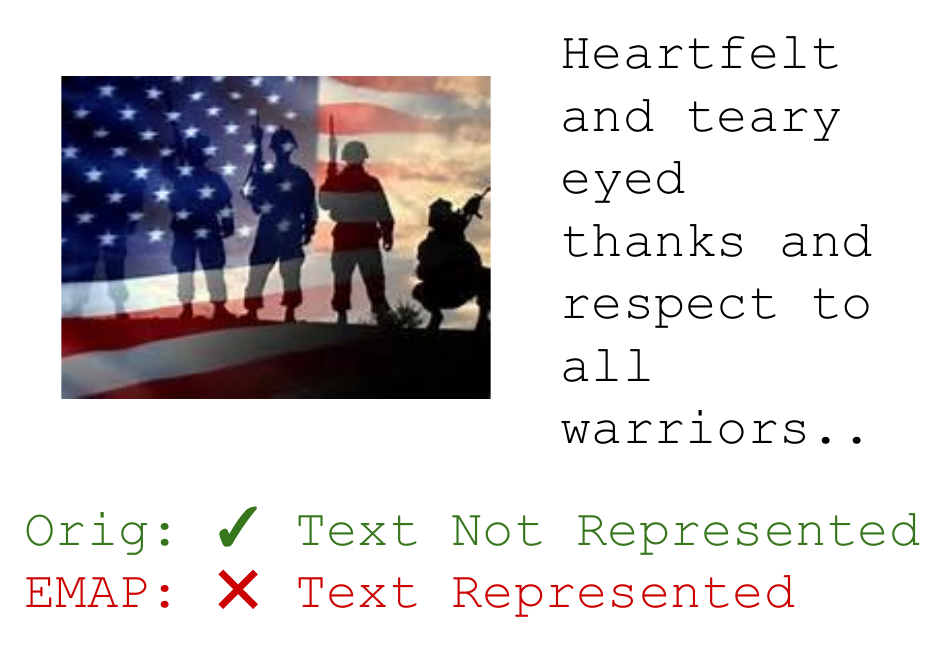}
  \hfill
  \includegraphics[width=.24\textwidth]{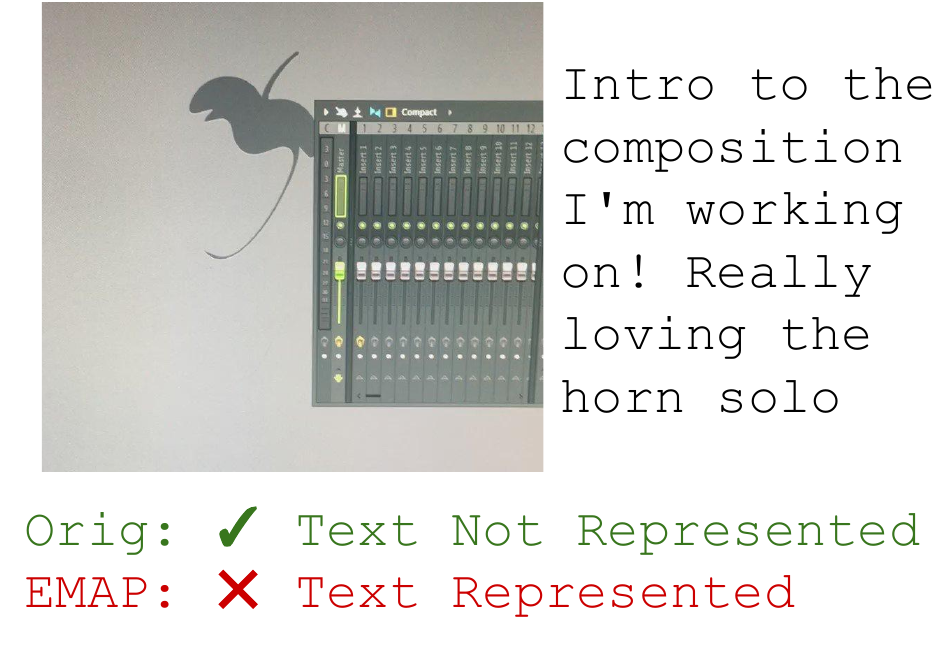}
  \hfill
  \caption{Examples of cases from for which \emap degrades the
    performance of LXMERT + Logits. All cases are labelled as ``image
    does not add meaning'' in the original corpus. Text of tweets may
    be gently modified for privacy reasons.}
  \label{fig:qual_examples}
\end{figure*}

\mparagraph{Computing Hardware} Linear models and feed-forward neural
networks were trained using consumer-level CPU/RAM configurations.
LXMERT was fine-tuned on single, consumer GPUs with 12GB of vRAM.

\mparagraph{Runtime} The slowest algorithm we used was LXMERT
\cite{tan2019lxmert}, and the biggest dataset we fine-tuned on was
R-POP. LXMERT was fine-tuned on the order of 1500 times. Depending on
the dataset, the fine-tuning process took between 10 minutes and an
hour. Overall, we estimate that we spent on the order of 50-100 GPU
days doing this work.

\mparagraph{Number of Parameters} The RoBERTa features we used were
2048-dimensional, and the EfficientNetB4 features were 1792
dimensional. The linear models and feed forward neural networks
operated on those. We cross-validated the number of layers and the
width, but the maximal model, for these cases, has on the order of
millions of parameters. The biggest model we used was LXMERT,
which has a comparable memory footprint to the original BERT Base
model.

\section{Qualitative Analysis of T-VIS}
\label{sec:qualitative_tvis}
To demonstrate the potential utility of \emap in qualitative
examinations, we identified the individual instances in T-VIS for
which \emap changes the test-time predictions of the LXMERT + Linear
Logits model. Recall that in this dataset, \emap hurts performance.

In introducing this task, \newcite{vempala2019categorizing} propose categorizing
image+text tweets into four categories: ``Some or all of the content
words in the text are represented in the image'' (or not) $\times$
``Image has additional content that represents the meaning of the text
and the image'' (or not).

As highlighted in \autoref{tab:consistency_results},
when \emap changes the prediction of the full model (14.5\% of cases),
the prediction made is  incorrect more often not: among label swapping cases, when the \emap or the original
model is correct, the original prediction is correct in 55\% of the 
cases. 

The most common label swaps of this form are between the classes:
``image does not add'' $\times$ \{``text is represented'', ``text is
not represented''\}; as shorthand for these two classes, we will write
IDTR (``image doesn't add, text represented'') and IDTN (``image
doesn't add, text not represented''). Across the 10 cross-validation
splits, \emap incorrectly maps the original model's correct prediction
of INTR $\rightarrow$ IDTN 255 times. For reference, there are 165
cases where \emap maps the \emph{incorrect} INTR prediction of the
original model to the correct IDTN label. So, when \emap makes the change
INTR $\rightarrow$ IDTN, in 60\% of cases the full model is
correct.  Similarly, \emap incorrectly maps the original model's
correct prediction of INTN $\rightarrow$ IDTR 77 times (versus 48
correct mappings, original model correct in 62\% of cases).

\autoref{fig:qual_examples} gives some instances from T-VIS where
images do not add meaning and the \emap projection causes a
swap from a correct to an incorrect prediction. While it's difficult
to draw conclusions from single instances, some preliminary patterns
emerge. There are a cluster of animal images coupled with captions
that may be somewhat difficult to map. In the case of the dog with
eyebrows, the image isn't centered on the animal, and it might be
difficult to identify the eyebrows without the prompt of the caption
(hence, interactions might be needed). Similarly, the bear-looking-dog
case is difficult: the caption doesn't explicitly mention a dog, and
the image itself depicts an animal that isn't canonically
canine-esque; thus, modeling interactions between the image and the
caption might be required to fully disambiguate the meaning.

\autoref{fig:qual_examples} also depicts two cases where the original model predicts that the text is not represented
(but \emap does).
They are drawn from a cluster of similar
cases where the captions seem indirectly connected to the image. 
Consider the $4^{{\rm th}}$, music composition
example: just looking at the text, one could envision a more literal
manifestation: i.e., a person playing a horn. Similarly, looking just
at the screenshot of the music production software, one could envision
a more literal caption, e.g., ``producing a new song with this software.''
But, the real connection is less direct, and might require additional
cross-modal inferences. Other common cases of INTR $\rightarrow$ IDTN
are ``happy birthday'' messages coupled with images of their intended
recipients and selfies taken at events (e.g., sports games), with
descriptions (but not direct visual depictions).

\section{Worked Example of \emap}
\label{sec:walkthrough}
We adopt the notation of \autoref{sec:theoretical} and give an
concrete worked example of \emap. Consider the following $f$
output values, which are computed on three image+text pairs for a binary
classification task. We assume that $f$ outputs an un-normalized logit
that can be passed to a scaling function like the logistic function
$\sigma$ for a probability estimate over the binary outcome, e.g.,
$f_{11} = -1.3$ and $\sigma(f_{11}) = P(y=1) \approx .21$.

\begin{tabular}{ccc}
  $f_{11} = -1.3$ & $f_{12} = .3$ & $f_{13} = -.2$ \\
  $f_{21} = .8$ & $f_{22} = 3$ & $f_{23} = 1.1$\\
  $f_{31} = 1.1$ & $f_{32} = -.1$ & $f_{33} = .7$
\end{tabular}

\noindent We can write this equivalently in matrix form:
$$
\begin{bmatrix}
  f_{11} & f_{12} & f_{13} \\
  f_{21} & f_{22} & f_{23} \\
  f_{31} & f_{32} & f_{33} 
\end{bmatrix} =
\begin{bmatrix}
  -1.3 & 0.3 & -0.2 \\
  0.8 & 3.0 & 1.1\\
  1.1 & -.1 & 0.7
\end{bmatrix} 
$$ Note that the mean logit predicted by $f$ over these 3 examples is
$.6$. We use \autoref{eq:vjopt} to compute the 
$\vparam$s; this is equivalent to taking a column-wise mean of this
matrix, which yields (approximately) $[.2, 1.07, .53]$. Similarly, we
can use \autoref{eq:tiopt}, equivalent to taking a row-wise mean
of this matrix, which yields (approximately) $[-.4, 1.63, .57]$, and
then subtract the overall mean $.6$ to achieve $[-1, 1.03,
  -.03]$. Finally, we can sum these two results to compute $[\hat
  f_{11}, \hat f_{22}, \hat f_{33}] = [-.8, 2.1, .5]$. These
predictions are the closest approximations to the full evaluations
$[f_{11}, f_{22}, f_{33}] = [-1.3, 3, .7]$ for which the generation function
obeys the additivity
constraint over the three input pairs.

\putbib[refs]
\end{bibunit}

\end{document}